\documentclass{article}

\usepackage{microtype}
\usepackage{graphicx}
\usepackage{subcaption}
\usepackage{adjustbox}
\usepackage{booktabs} % for professional tables

\usepackage{hyperref}

\usepackage[preprint]{icml2026}

\usepackage{amsmath}
\usepackage{amssymb}
\usepackage{mathtools}
\usepackage{amsthm}
\usepackage{multirow}

\usepackage[capitalize,noabbrev]{cleveref}

\theoremstyle{plain}
\newtheorem{theorem}{Theorem}[section]

\theoremstyle{definition}
\newtheorem{definition}[theorem]{Definition}
\newtheorem{assumption}[theorem]{Assumption}
\theoremstyle{remark}

\usepackage[textsize=tiny]{todonotes}

\icmltitlerunning{Controlled-Minimax: Stabilizing Adversarial Optimization via Control Theory}

\begin{document}

\twocolumn[
  \icmltitle{Controlled-Minimax: Stabilizing Adversarial Optimization via Control Theory}


  \icmlsetsymbol{equal}{*}

  \begin{icmlauthorlist}
    \icmlauthor{Tianjiao Luo}{thu}
    \icmlauthor{Jianfei Chen}{thu}
    \icmlauthor{Jun Zhu}{thu}
  \end{icmlauthorlist}

  \icmlaffiliation{thu}{Department of Computer Science and Technology, Tsinghua University, Beijing, China}

  \icmlcorrespondingauthor{Jun Zhu}{dcszj@mail.tsinghua.edu.cn}

  \icmlkeywords{Machine Learning, ICML}

  \vskip 0.3in
]



\printAffiliationsAndNotice{}  % no special notice (required even if empty)

\begin{abstract}
  Minimax optimization is fundamental to adversarial learning yet suffers from severe stability challenges, often manifesting as persistent limit cycles or divergence. This paper investigates these instabilities through the lens of continuous-time dynamical systems. By modeling training dynamics as a second-order Ordinary Differential Equation (ODE), we rigorously identify the root cause of non-convergence: the problematic eigenstructure of the loss function's Hessian near equilibrium points. To address this, we introduce Controlled-Minimax, a novel stabilization framework that reformulates the optimization trajectory as a controllable physical system. Unlike heuristic regularization techniques, our approach derives a specific feedback controller that provides theoretical guarantees of local exponential stability, regardless of the underlying Hessian spectrum. Consequently, our method stabilizes general minimax problems where traditional gradient descent-ascent methods fail. Extensive experiments on pathological benchmarks including the Dirac-GAN, 2D Gaussian mixtures, and Swiss Roll datasets, as well as high-resolution image synthesis with StyleGAN2-ada, demonstrate that Controlled-Minimax effectively suppresses oscillatory behavior while ensuring fast and robust convergence.
\end{abstract}

\section{Introduction}
Minimax optimization is a cornerstone of modern machine learning, transcending traditional single-objective minimization tasks to address problems characterized by competing interests. This paradigm frames learning as a competition between two agents: a minimizing player controlling parameters $\mathbf{\theta}$ and a maximizing player controlling $\mathbf{\phi}$, coupled through a shared objective function $L(\mathbf{\theta}, \mathbf{\phi})$. The goal is to find a saddle point $(\mathbf{\theta}^*, \mathbf{\phi}^*)$ that solves the problem $\min_{\mathbf{\theta}} \max_{\mathbf{\phi}} L(\mathbf{\theta}, \mathbf{\phi})$. This formulation has proven to be incredibly powerful and versatile, underpinning the success of Generative Adversarial Networks (GANs) in synthesizing high-fidelity data \citep{goodfellow2020generative}, enhancing model robustness through adversarial training \citep{madry2017towards}, and enabling complex decision-making in multi-agent reinforcement learning \citep{busoniu2008comprehensive}.

However, efficiently solving these minimax problems in the context of deep learning remains a formidable challenge. Unlike classical convex-concave settings, the optimization landscapes parameterized by deep neural networks are highly non-convex and non-concave, leading to complex local geometries. Standard approaches, most notably Gradient Descent-Ascent (GDA) \citep{bottou2010large}, often fail to converge to a stable Nash equilibrium in such regimes. Instead of simply descending a potential surface, the optimization dynamics frequently exhibit rotational behaviors caused by the interaction between the two players. Consequently, training trajectories are prone to persistent oscillations, limit cycles, or even divergence, which severely hampers the stability and generalization capability of the resulting models.

To address the inherent instability of minimax optimization, existing literature has predominantly explored two directions: refining the optimization dynamics and imposing structural regularizations. The former category attempts to correct the rotational behavior of GDA by incorporating lookahead mechanisms \citep{korpelevich1976extragradient, mertikopoulos2018optimistic, gidel2019negative}, multi-step updates \citep{goodfellow2020generative}, or timescale separation between the players \citep{heusel2017gans}. The latter seeks to smooth the optimization landscape by enforcing Lipschitz constraints or adding gradient penalties \citep{gulrajani2017improved, mescheder2018training, miyato2018spectral}. Nevertheless, rigorous theoretical guarantees for these methods remain elusive, particularly in the non-convex non-concave settings typical of neural networks in deep learning. 

While local convergence can be proven under strict assumptions—such as local convexity or linearity \citep{nagarajan2017gradient, heusel2017gans}, global convergence properties are often neither guaranteed nor understood \citep{mescheder2017numerics, daskalakis2018limit, farnia2020gans}. Furthermore, regularization-based methods tend to introduce inductive biases that hamper model expressivity \citep{mescheder2018training, arjovsky2017wasserstein}, while sophisticated dynamic adjustments frequently incur high computational overheads without ensuring stability in general cases. Consequently, developing a solver that is both theoretically grounded and practically robust remains an open challenge.

In this paper, we depart from the traditional view of designing better discrete optimizers and instead approach the instability problem through the lens of continuous-time control theory. We model the minimax training dynamics as a second-order Ordinary Differential Equation (ODE), allowing us to rigorously diagnose the root cause of persistent oscillations: the specific eigenstructure of the loss function’s Hessian, which inherently precludes convergence in standard settings. Based on this insight, we propose \textbf{Controlled-Minimax}, a novel framework that treats the optimization trajectory as a controllable dynamical system. Unlike heuristic regularizers that passively penalize model weights, our method actively injects a designed feedback controller into the dynamics. This controller forces the system’s energy to dissipate, providing a theoretical guarantee of \emph{local exponential stability} regardless of the underlying Hessian structure. We validate our approach on tasks ranging from the pathological Dirac-GAN, 2D Gaussian and Swiss Roll to high-resolution image generation with StyleGAN2-ada, demonstrating that Controlled-Minimax effectively eliminates oscillatory behavior and ensures robust convergence where traditional methods fail. 

Our contribution can be summarized as:
\begin{itemize}
    \item We bridge the gap between discrete optimization algorithms and continuous dynamical systems by transforming the adversarial training process into a second-order Ordinary Differential Equation (ODE). This shift in perspective allows us to apply classical tools from control theory to the typically unstable dynamics of GAN training.
    \item We identify that the instability in adversarial learning stems from specific pathological eigenstructures of the Hessian that standard gradient fields cannot correct.
    \item We introduce a novel stabilization mechanism called the Controlled-Minimax that views optimization as a control problem. By designing a specific feedback controller on gradient, we prove that our method achieves local exponential stability for general minimax problems, offering a rigorous guarantee that is absent in heuristic regularization methods.
    \item We demonstrate the effectiveness of our approach across a spectrum of tasks, ranging from fundamental benchmarks (Dirac-GAN, 2D Gaussian, Swiss Roll) to high-dimensional training with StyleGAN2-ada. Our Controlled-Minimax consistently eliminates oscillations and accelerates convergence, validating its practical effectiveness alongside its theoretical soundness.
\end{itemize}

\subsection{Related Work}

\paragraph{Convergence Analysis and Static Bounds}
Many works on minimax optimization relies on strong assumptions regarding the curvature of the objective function. Convergence rates have been established for settings such as convex-concave \citep{nemirovski2009robust}, strongly-convex-strongly-concave \citep{palaniappan2016stochastic}, and specific nonconvex-nonconcave case \citep{yang2020global}. Within these structured landscapes, \citet{daskalakis2018limit} and \citet{gidel2019negative} provide rigorous analysis of Gradient Descent-Ascent (GDA), though requiring specific step-size conditions or negative momentum. Similarly, \citet{mertikopoulos2018optimistic} explore extra-gradient methods in coherent games, while \citet{hsieh2019finding} analyze convergence in non-convex non-concave settings. However, these guarantees falter in deep learning because they depend on rigid curvature profiles or specific timescale separations \citep{mescheder2017numerics, heusel2017gans}. Furthermore, treating optimization steps as isolated updates, ignoring the accumulated trajectory and inertia of the parameters.

\paragraph{Heuristic Regularization and Algorithmic Variants}
 To address the rotational dynamics common in zero-sum games, \citet{daskalakis2018limit} propose the Optimistic Mirror Descent, which effectively utilizes negative momentum to cancel out interaction terms, a technique further analyzed by \citet{mokhtari2020unified} through the lens of proximal point approximations. Regularization strategies have also been pivotal. \citet{nagarajan2017gradient} analyze the local convergence properties of gradient penalties, introducing terms that stabilize the Jacobian of the vector field near equilibrium points. However, a key limitation of these approaches is that their convergence guarantees often rely on local linearity or strong operator assumptions. 

\paragraph{Control Theory and Dynamical Systems}
\citet{xu2020understanding} and \citet{luo2023stabilizing} have attempted to actively control training dynamics using feedback loops. \citet{chavdarova2019reducing} further explored how differential equation solvers can reduce noise in game dynamics. Furthermore, \citet{li2019convergence} and \citet{su2016differential} have drawn connections between Nesterov acceleration and differential equations, though primarily in single-objective settings. While these works acknowledge the dynamic nature of learning, they often rely on complex auxiliary loss terms or heavy modifications to the standard gradient vector field. This can make them computationally expensive or difficult to integrate with standard optimizers.

% \begin{figure*}[t]
%     \centering
%     \begin{minipage}{0.48\textwidth}
%         \centering
%         \includegraphics[width=\linewidth]{img/phase_portrait.pdf}
%         % \caption{Phase Portrait Analysis}
%         \label{fig:phase_portrait}
%     \end{minipage}\hfill
%     \begin{minipage}{0.48\textwidth}
%         \centering
%         \includegraphics[width=\linewidth]{img/time_series.pdf}
%         % \caption{Time Series Evolution}
%         \label{fig:time_series}
%     \end{minipage}
%     \caption{Combined visualization showing the phase portrait (left) and time series (right) for the controlled system.}
%     \label{fig:combined_plots}
% \end{figure*}

\begin{figure}[t]
    \centering
    % First image (Top)
    \includegraphics[width=\linewidth]{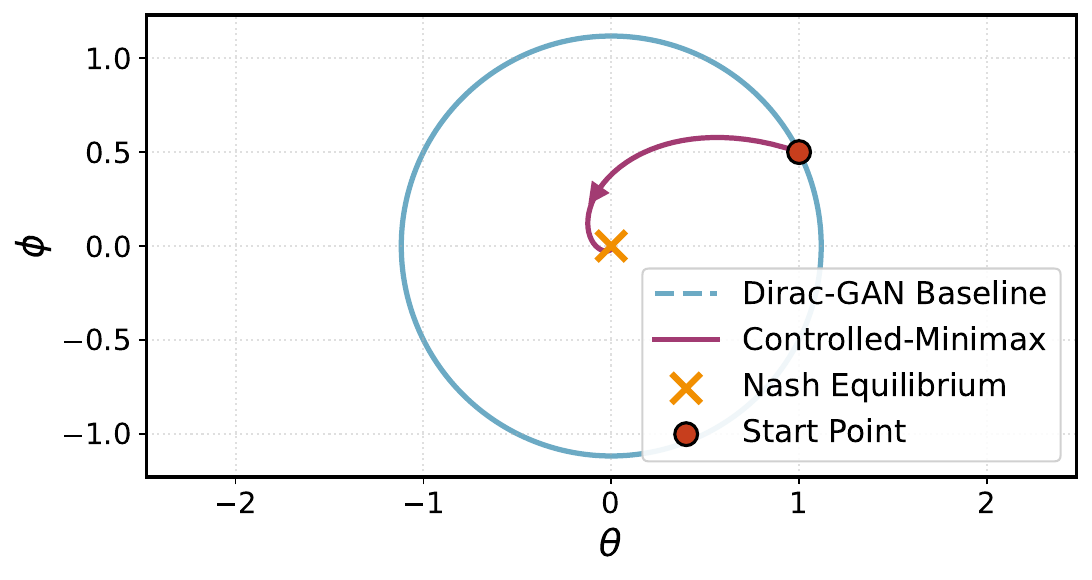}
    
    % Small vertical gap between images
    \par\medskip
    
    % Second image (Bottom)
    \includegraphics[width=\linewidth]{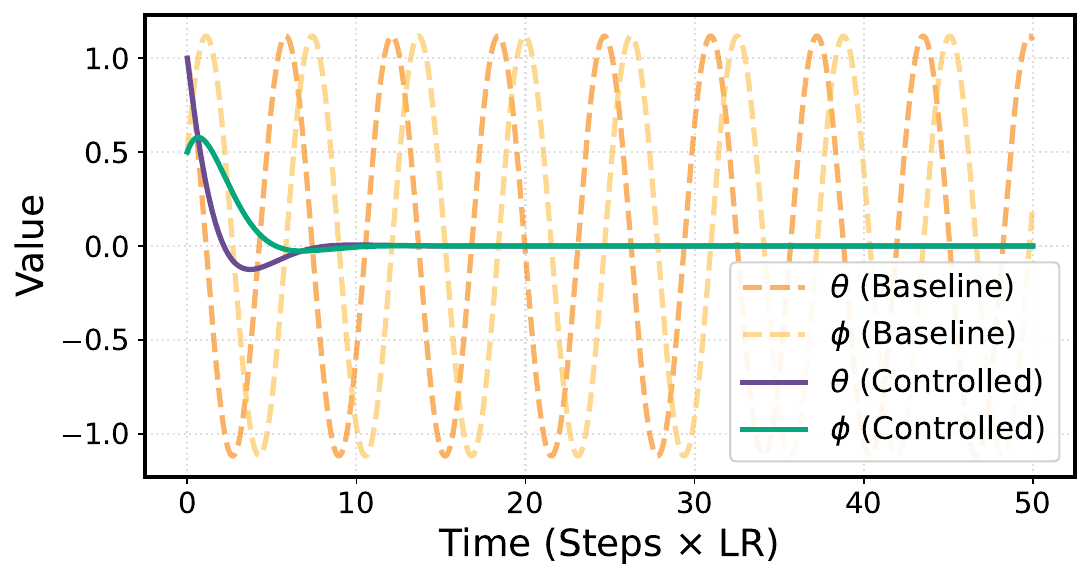}

    \caption{Convergence analysis on Dirac-GAN. \textbf{Top:} Phase space trajectories showing the optimization path. The baseline gets stuck in a limit cycle, while our Controlled Minimax spirals inward to the equilibrium. \textbf{Bottom:} Evolution of parameters over time. The baseline parameters oscillate indefinitely with constant amplitude, whereas our method forces the parameter values to decay to zero.}
    
    \label{fig:combined_plots}
\end{figure}

\section{Preliminary}
We collect here standard notions and results on local stability of
ordinary differential equations (ODEs) that will be used in the
analysis of our second-order dynamics. Throughout, we consider
autonomous systems of the form
\begin{equation}
    \dot{x} = f(x),
    \qquad x \in \mathbb{R}^d,
    \label{eq:general-ode}
\end{equation}
where $f : \mathbb{R}^d \to \mathbb{R}^d$ is locally Lipschitz.

\begin{definition}[Equilibrium]
    A point $x^\star \in \mathbb{R}^d$ is an \emph{equilibrium} of system
\eqref{eq:general-ode} if $f(x^\star) = 0$. We recall the standard
stability notions.
\end{definition}

\begin{definition}
\label{def:stability}
System \eqref{eq:general-ode} with equilibrium  $x^\star$ is 
\begin{enumerate}
    \item \emph{(Lyapunov) stable} if for every
    $\varepsilon > 0$ there exists $\delta > 0$ such that
    $\|x(0) - x^\star\| < \delta$ implies
    $\|x(t) - x^\star\| < \varepsilon$ for all $t \ge 0$.

    \item \emph{asymptotically stable} if it is stable and,
    in addition, there exists $\delta' > 0$ such that
    $\|x(0) - x^\star\| < \delta'$ implies
    $\lim_{t \to \infty} x(t) = x^\star$.

    \item \emph{exponentially stable} if there exist
    constants $c > 0$, $\alpha > 0$, and $r > 0$ such that
    $\|x(0) - x^\star\| < r$ implies
    \[
        \|x(t) - x^\star\|
        \le c\,e^{-\alpha t}\,\|x(0) - x^\star\|
        \quad \text{for all } t \ge 0.
    \]
\end{enumerate}
\end{definition}

We recall the classical local linearization result, which relates the
stability of an equilibrium of \eqref{eq:general-ode} to the spectrum of
the Jacobian of $f$ at that equilibrium.

\begin{theorem}[Local Linearization]
\label{thm:local-linearization}
Let $x^\star$ be an equilibrium of
$
    \dot{x} = f(x),
$
and Jacobian $A \coloneqq Df(x^\star)$.

\begin{enumerate}
    \item If all eigenvalues of $A$ satisfy $\Re(\lambda) < 0$, then
    $x^\star$ is \emph{locally asymptotically stable}.

    \item If at least one eigenvalue of $A$ satisfies $\Re(\lambda) > 0$,
    then $x^\star$ is \emph{unstable}.

    \item If all eigenvalues of $A$ satisfy $\Re(\lambda) \le 0$ and for every eigenvalue with $\Re(\lambda) = 0$ the algebraic multiplicity equals the geometric multiplicity, then $x^\star$ is \emph{marginally stable}.
\end{enumerate}
\end{theorem}

\section{Minimax Dynamics as a Second Order ODE}
\label{sec:setup}

% \subsection{Gradient Flow and Second-Order Dynamics}
% \label{subsec:second-order-from-gf}
We consider a two-player zero-sum minimax problem 
\begin{equation}
    \min_{\theta \in \mathbb{R}^m} \; \max_{\phi \in \mathbb{R}^n} 
    L(\theta, \phi),
    \label{eq:minimax-problem}
\end{equation}
where $L : \mathbb{R}^m \times \mathbb{R}^n \to \mathbb{R}$ is a continuously differentiable payoff function. The vector $\theta \in \mathbb{R}^m$ denotes the strategy of the minimizing player, and $\phi \in \mathbb{R}^n$ denotes the strategy of the maximizing player.
A standard continuous-time model for gradient
descent--ascent is the gradient flow
\begin{equation}
   \dot{\theta} = -\nabla_{\theta} L(\theta,\phi),
    \qquad
    \dot{\phi} = \nabla_{\phi} L(\theta,\phi),
    \label{eq:gf-theta-phi}
\end{equation}
where, with a slight abuse of notation, we write \(\theta,\phi\) instead of
\(\theta(t),\phi(t)\) to denote the time-dependent trajectories.

% We collect the player variables into a single state vector
% \begin{equation}
%     z \coloneqq
%     \begin{bmatrix}
%         \theta \\[2pt]
%         \phi
%     \end{bmatrix}
%     \in \mathbb{R}^{m+n},
%     \qquad
%     v \coloneqq \frac{dz}{dt}
%     =
%     \begin{bmatrix}
%         \dfrac{d\theta}{dt} \\[4pt]
%         \dfrac{d\phi}{dt}
%     \end{bmatrix},
%     \label{eq:z-v-def}
% \end{equation}
% so that \eqref{eq:gf-theta-phi} can be written compactly as
% \begin{equation}
%     \frac{dz}{dt} = v,
%     \qquad
%     v =
%     \begin{bmatrix}
%         -\nabla_{\theta} L(\theta,\phi) \\[2pt]
%         \nabla_{\phi} L(\theta,\phi)
%     \end{bmatrix}.
%     \label{eq:gf-compact}
% \end{equation}

To obtain a second-order ODE, we differentiate
\eqref{eq:gf-theta-phi} with respect to time. Using the chain rule, we
have
\begin{equation}
    \ddot\theta
    = -\frac{d}{dt}\,\nabla_{\theta} L(\theta,\phi),
    \qquad
    \ddot\phi
    = \frac{d}{dt}\,\nabla_{\phi} L(\theta,\phi).
    \label{eq:theta-phi-second-deriv}
\end{equation}
Assuming that \(L\) is twice continuously differentiable, the right-hand
sides of \eqref{eq:theta-phi-second-deriv} can be written as
\begin{equation}
\left\{
\begin{aligned}
      \ddot\theta
    &=
    -\nabla^2_{\theta\theta} L(\theta,\phi)\,\dot{\theta}
    -\nabla^2_{\theta\phi} L(\theta,\phi)\,\dot{\phi},
    \\
    \ddot\phi
    &=
    \nabla^2_{\phi\theta} L(\theta,\phi)\,\dot{\theta}
    +\nabla^2_{\phi\phi} L(\theta,\phi)\,\dot{\phi}.
\end{aligned}
\right.
\label{eq:theta-phi-second-order-system}
\end{equation}
where
\(\nabla^2_{\theta\theta} L, \nabla^2_{\theta\phi} L,
\nabla^2_{\phi\theta} L, \nabla^2_{\phi\phi} L\) denote the corresponding
second derivatives of \(L\). This formulation represents the uncontrolled second-order dynamics. In Section \ref{sec:ori_stability}, we provide a stability analysis of this original system, followed by the design of a stabilizing structure-aware damping controller in Section \ref{sec:control}.

\subsection{Local Stability Analysis} \label{sec:ori_stability}

In this section, we analyze the local stability of the dynamical system derived in Equation \eqref{eq:theta-phi-second-order-system} and formally derive the conditions for three cases of stability in theorem \ref{thm:stability_ori}.

Let the parameter spaces be $\theta \in \mathbb{R}^{n_\theta}$ and $\phi \in \mathbb{R}^{n_\phi}$. To conserve space, we define the full state vector $\mathbf{u} \in \mathbb{R}^{2(n_\theta + n_\phi)}$ using transpose notation $
    \mathbf{u} = \left[ \theta^\top, \phi^\top, \dot{\theta}^\top, \dot{\phi}^\top \right]^\top.
$
Then the dynamics can be transformed to 
\begin{equation}
     \dot{\mathbf{u}} = 
    \begin{bmatrix}
    \dot{\theta} \\
    \dot{\phi} \\
    -\nabla^2_{\theta\theta} L(\theta, \phi) \dot{\theta} - \nabla^2_{\theta\phi} L(\theta, \phi) \dot{\phi} \\
    \nabla^2_{\phi\theta} L(\theta, \phi) \dot{\theta} + \nabla^2_{\phi\phi} L(\theta, \phi) \dot{\phi}
    \end{bmatrix}.
\end{equation}

To determine the long-term convergence of the training process, we seek fixed points where the system ceases to evolve. An equilibrium is defined by the condition $\dot{\mathbf{u}} = \mathbf{0}$, requiring both that the parameter updates halt and that the loss gradients vanish. We formalize this stationary state as follows.

\begin{definition}[Zero-Velocity Equilibrium]
An equilibrium point of the system is defined as a state $\mathbf{u}^*$ where the time derivatives of the parameters vanish. We denote this state as:
\begin{equation}
    \mathbf{u}^* = \left[ (\theta^*)^\top, (\phi^*)^\top, \mathbf{0}^\top, \mathbf{0}^\top \right]^\top,
\end{equation}
where $\theta^*$ and $\phi^*$ are fixed points in the parameter space, and $\mathbf{0}$ represents the zero vector of appropriate dimension.
\end{definition}

To determine local stability, we examine the Jacobian matrix of $\mathcal{F}$ evaluated at $\mathbf{u}^*$. Computing the partial derivatives and noting that Hessian derivative terms vanish at zero velocity, we obtain the block-structured matrix:
\begin{equation}
    \mathbf{J}(\mathbf{u}^*) = 
    \begin{bmatrix}
    \mathbf{0} & \mathbf{0} & \mathbf{I} & \mathbf{0} \\
    \mathbf{0} & \mathbf{0} & \mathbf{0} & \mathbf{I} \\
    \mathbf{0} & \mathbf{0} & -\nabla^2_{\theta\theta} L & -\nabla^2_{\theta\phi} L \\
    \mathbf{0} & \mathbf{0} & \nabla^2_{\phi\theta} L & \nabla^2_{\phi\phi} L
    \end{bmatrix},
\end{equation}
where $L$ is evaluated at $(\theta^*, \phi^*)$. The stability is determined by the lower-right block, defined as the interaction matrix $\mathbf{M}$:
\begin{equation}
    \mathbf{M} = 
    \begin{bmatrix}
    -\nabla^2_{\theta\theta} L(\theta^*, \phi^*) & -\nabla^2_{\theta\phi} L(\theta^*, \phi^*) \\
    \nabla^2_{\phi\theta} L(\theta^*, \phi^*) & \nabla^2_{\phi\phi} L(\theta^*, \phi^*)
    \end{bmatrix}.
\end{equation}

\begin{theorem}[Local Stability Classification] \label{thm:stability_ori}
Let $\sigma(\mathbf{M})$ denote the spectrum of the interaction matrix $\mathbf{M}$. The local stability of the system \eqref{eq:theta-phi-second-order-system} around $\mathbf{u}^*$ is:

\begin{enumerate}
    \item  Asymptotically Stable if $\forall \lambda \in \sigma(\mathbf{M}), \text{Re}(\lambda) < 0$. Small perturbations in velocity decay exponentially.
    
    \item Lyapunov Stable (Marginally Stable) if $\forall \lambda \in \sigma(\mathbf{M}), \text{Re}(\lambda) \leq 0$, and for any eigenvalue with $\text{Re}(\lambda) = 0$ the algebraic multiplicity equals the geometric multiplicity. Velocity perturbations remain bounded.
    
    \item  Unstable if $\exists \lambda \in \sigma(\mathbf{M})$ such that $\text{Re}(\lambda) > 0$. Arbitrarily small perturbations will grow exponentially over time.
\end{enumerate}
\end{theorem}

\section{Controlled Dynamics Formulation} \label{sec:control}
In this section, we design a feedback controller for the second-order dynamics derived in Equation \eqref{eq:theta-phi-second-order-system}. We provide a rigorous proof that the resulting system achieves local exponential stability, ensuring robust convergence to the equilibrium. Building on this theoretical foundation, we propose the Controlled Minimax algorithm, which integrates our control strategy into a practical optimization framework.

% Having established the local stability properties of the intrinsic second-order dynamics in Theorem \ref{thm:stability_ori}, we now propose a control framework to actively modify the system's convergence behavior. 
To facilitate the design of the controller, we extend the natural dynamics derived in \eqref{eq:theta-phi-second-order-system} by introducing auxiliary control inputs $u_\theta$ and $u_\phi$. For notational compactness, we denote the Hessian blocks of $L$ as $H_{\theta\theta} = \nabla^2_{\theta\theta} L$, $H_{\theta\phi} = \nabla^2_{\theta\phi} L$, and so forth. The resulting controlled dynamical system is given by:
\begin{equation}
\label{eq:controlled-system}
\begin{aligned}
    \ddot\theta
    &= - H_{\theta\theta}(\theta,\phi)\,\dot\theta
       - H_{\theta\phi}(\theta,\phi)\,\dot\phi
       + u_\theta(\theta,\phi,\dot\theta,\dot\phi),\\[2pt]
    \ddot\phi
    &= \phantom{-} H_{\phi\theta}(\theta,\phi)\,\dot\theta
       + H_{\phi\phi}(\theta,\phi)\,\dot\phi
       + u_\phi(\theta,\phi,\dot\theta,\dot\phi).
\end{aligned}
\end{equation}

We focus on the behavior of this system near a stationary point $(\theta^\ast,\phi^\ast)$ satisfying the first-order optimality conditions $\nabla_\theta L(\theta^\ast,\phi^\ast)=0$ and $\nabla_\phi L(\theta^\ast,\phi^\ast)=0$. To ensure the problem is well-posed for control synthesis, we introduce the following regularity assumptions.

\begin{assumption}
\label{ass:bounded-hessian}
The objective function $L$ is twice continuously differentiable ($L\in C^2$). Furthermore, there exist constants $r>0$ and $C>0$ such that for all states in the neighborhood
\begin{equation*}
    (\theta,\phi)\in\mathcal{N}
    :=\{(\theta,\phi):\|(\theta,\phi)-(\theta^\ast,\phi^\ast)\|\le r\},
\end{equation*}
the Hessian blocks satisfy the uniform bound:
% \begin{equation*}
%     \|H_{\theta\theta}(\theta,\phi)\|,
%     \ \|H_{\theta\phi}(\theta,\phi)\|,
%     \ \|H_{\phi\theta}(\theta,\phi)\|,
%     \ \|H_{\phi\phi}(\theta,\phi)\|
%     \le C,
% \end{equation*}
\begin{equation*}
\begin{aligned}
    &\|H_{\theta\theta}(\theta,\phi)\|,
    \ \|H_{\theta\phi}(\theta,\phi)\| \le C,\\
    &\|H_{\phi\theta}(\theta,\phi)\|,
    \ \|H_{\phi\phi}(\theta,\phi)\|
    \le C,
\end{aligned}
\end{equation*}
where $\|\cdot\|$ denotes the spectral norm (or any matrix norm induced by the Euclidean vector norm).
\end{assumption}

\begin{assumption}
\label{ass:well-posed}
We assume system ~\eqref{eq:theta-phi-second-order-system} is locally Lipschitz continuous in the aggregate state $(\theta,\phi,\dot\theta,\dot\phi)$ on $\mathcal{N}$. Consequently, solutions to this system exist and are unique in a neighborhood of the equilibrium $(\theta^\ast,\phi^\ast,0,0)$.
\end{assumption}

\subsection{Controller Design}

We propose a linear velocity feedback controller designed to introduce damping and cross-coupling between the $\theta$ and $\phi$ dynamics. The controllers are given by:
\begin{equation}
\label{eq:controller}
\begin{aligned}
    u_\theta(\dot\theta,\dot\phi)
    &= -\lambda_\theta \dot\theta
       + \beta A_c \dot\phi, \\[2pt]
    u_\phi(\dot\theta,\dot\phi)
    &= -\beta A_c^\top \dot\theta
       - \lambda_\phi \dot\phi,
\end{aligned}
\end{equation}
where $\lambda_\theta, \lambda_\phi > 0$ are scalar damping gains, $\beta \in \mathbb{R}$ is a cross-coupling gain, and $A_c \in \mathbb{R}^{d_\theta \times d_\phi}$ is a fixed coupling matrix.

In our settings, we define the effective rank $r \coloneqq \min\{d_\theta, d_\phi\}$ and select $A_c$ to provide identity coupling on the leading coordinates:
\begin{equation*}
    (A_c)_{ij} =
    \begin{cases}
        1, & \text{if } 1 \le i = j \le r,\\
        0, & \text{otherwise}.
    \end{cases}
\end{equation*}
This choice ensures direct interaction between the first $r$ components of $\theta$ and $\phi$ while leaving any remaining dimensions uncoupled.

To provide intuition for the proposed control structure, we examine the local behavior of the controlled system~\eqref{eq:controlled-system} near an equilibrium point $(\theta^\ast, \phi^\ast)$. Let $v(t) \coloneqq (\dot\theta(t), \dot\phi(t))^\top$ denote the combined velocity vector. The linearized velocity dynamics are governed by $\dot v = M v,$
where the closed-loop system matrix $M$ takes the form
\begin{equation}
\label{eq:M-matrix}
    M =
    \begin{pmatrix}
        - H_{\theta\theta}^\ast - \lambda_\theta I & - H_{\theta\phi}^\ast + \beta A_c \\[4pt]
        H_{\phi\theta}^\ast - \beta A_c^\top & H_{\phi\phi}^\ast - \lambda_\phi I
    \end{pmatrix}.
\end{equation}
Here, $H_{\alpha\gamma}^\ast \coloneqq H_{\alpha\gamma}(\theta^\ast, \phi^\ast)$ denotes the block components of the Hessian evaluated at the equilibrium.

The structure of $M$ illustrates the two distinct mechanisms introduced by the controller. First, the terms $-\lambda_\theta I$ and $-\lambda_\phi I$ provide diagonal damping to dissipate energy. Second, the term $\beta A_c$ introduces a skew-symmetric component that modifies the natural coupling between $\theta$ and $\phi$ (represented by the off-diagonal Hessian blocks). This decomposition highlights how the controller augments the inherent geometry of the system. We provide a rigorous stability proof for the full nonlinear system in the following section.

\subsection{Exponential Stability of the Controlled System}
In this section, we establish the stability properties of the closed-loop system \eqref{eq:controlled-system} under the proposed controller \eqref{eq:controller}.  Our primary focus is to determine the conditions on the damping gains $\lambda_\theta$ and $\lambda_\phi$ that guarantee the convergence of the velocity dynamics to zero. By analyzing the spectral properties of the system matrix $M$ defined in \eqref{eq:M-matrix}, we derive a sufficient condition for exponential stability that is robust to the coupling parameter $\beta$. The following theorem formalizes this result, providing an explicit decay rate for the velocity vector $v(t)$.

\begin{theorem}[Local Exponential stability]
\label{thm:local-stability}
Suppose Assumptions~\ref{ass:bounded-hessian}--\ref{ass:well-posed} hold and let $C>0$ be as in Assumption~\ref{ass:bounded-hessian}. Consider the closed-loop system~\eqref{eq:controlled-system} under the controller~\eqref{eq:controller}. If the damping gains are chosen such that
\begin{equation}
\label{eq:lambda-condition}
    \gamma := \min\{\lambda_\theta, \lambda_\phi\} - 2C > 0,
\end{equation}
then for any $\beta\in\mathbb{R}$, 
\[
    \|v(t)\| \le e^{-\gamma t}\,\|v(0)\|\qquad \text{for all } t\ge 0.
\]
Consequently, the full nonlinear system is locally exponentially stable at the equilibrium $(\theta^\ast,\phi^\ast,0,0)$.
\end{theorem}

While Theorem~\ref{thm:local-stability} provides a solid theoretical guarantee for the convergence of the closed-loop system, it relies on the assumption that the control input is available continuously and instantaneously. In a realistic experimental setup, the control law must be discretized and computed iteratively. Therefore, we now turn our attention to the algorithmic implementation. Section~\ref{sec:algorithm} details the specific iterative scheme used to approximate the optimal control inputs and discusses its convergence properties and computational complexity.

\begin{figure*}[t] % The '*' makes it span both columns. Remove '*' for a single column.
    \centering
    
    % --- First Subfigure ---
    \begin{subfigure}[b]{0.32\textwidth}
        \centering
        \includegraphics[width=\linewidth]{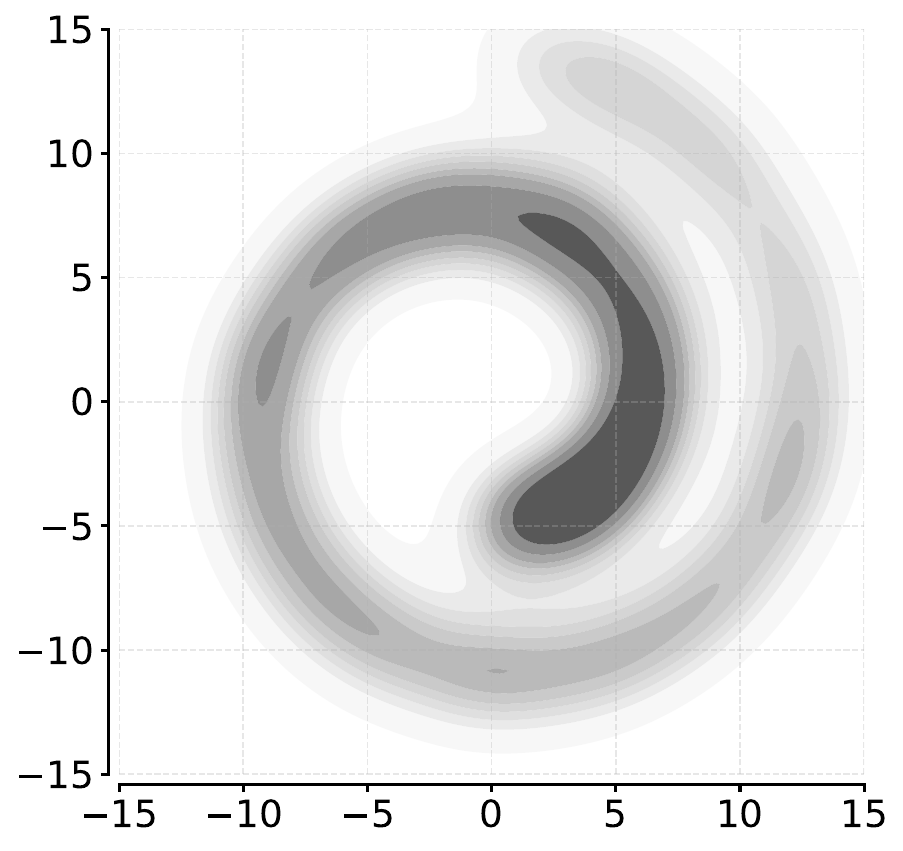}
        \caption{Target Distribution}
        \label{fig:target}
    \end{subfigure}
    \hfill % Adds flexible space between images
    % --- Second Subfigure ---
    \begin{subfigure}[b]{0.32\textwidth}
        \centering
        \includegraphics[width=\linewidth]{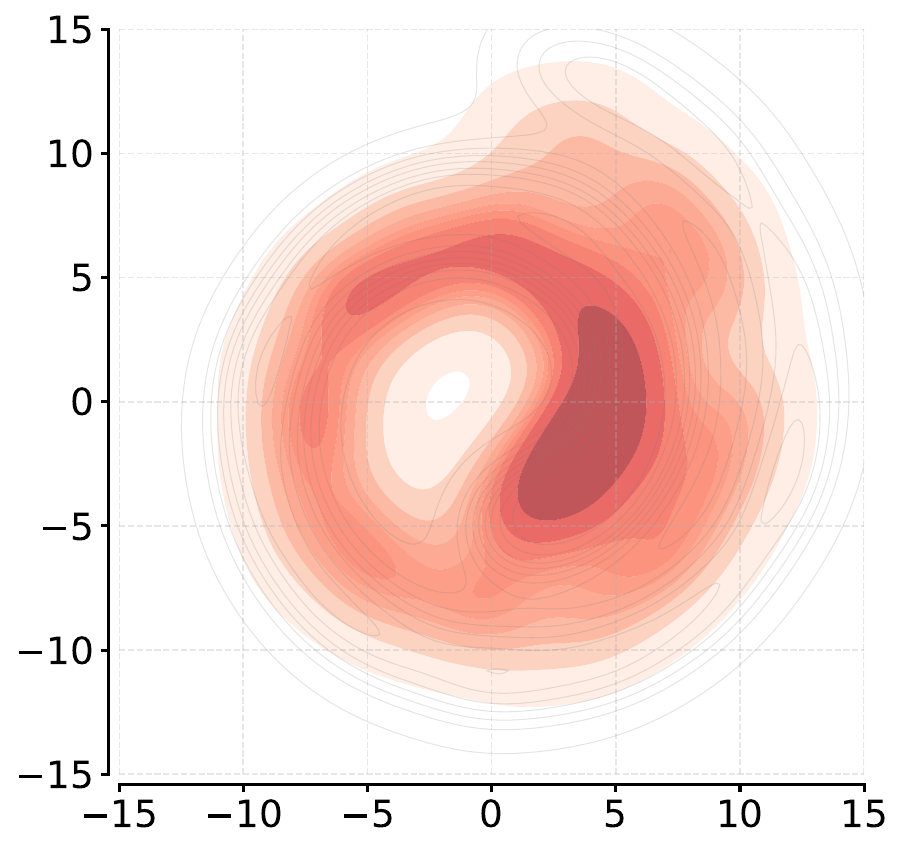}
        \caption{Baseline (Standard GAN)}
        \label{fig:baseline}
    \end{subfigure}
    \hfill % Adds flexible space between images
    % --- Third Subfigure ---
    \begin{subfigure}[b]{0.32\textwidth}
        \centering
        \includegraphics[width=\linewidth]{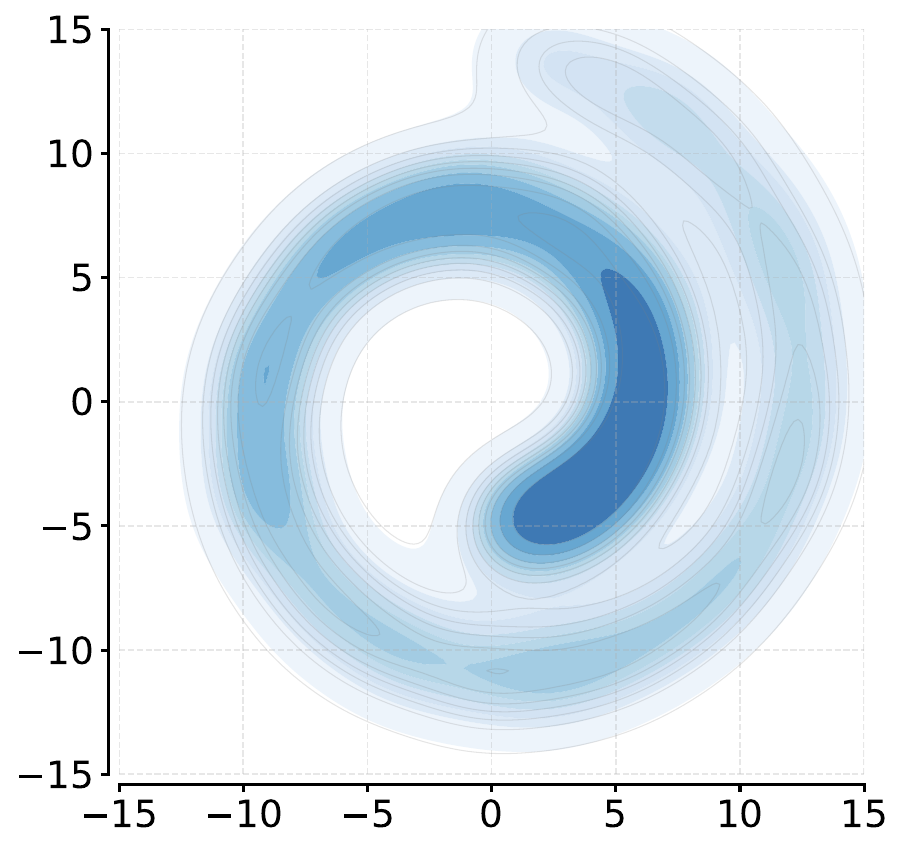}
        \caption{Ours (Controlled-Minimax)}
        \label{fig:ours}
    \end{subfigure}
    
    \caption{Qualitative results on the Swiss Roll dataset. While the standard GAN baseline struggles to learn the full manifold structure and exhibits disjointed segments, our proposed Controlled-Minimax successfully generates a continuous and smooth spiral that closely matches the target distribution.}
    \label{fig:comparison_swiss}
\end{figure*}

\section{The Controlled-Minimax Algoirthm} \label{sec:algorithm}
In this section, we propose the \textit{Controlled-Minimax} algorithm to solve the minimax optimization problem while ensuring robust stability. This method augments standard gradient descent-ascent dynamics by integrating the cross-coupling controller design defined in \eqref{eq:controller}. By substituting this controllers into the update steps, the algorithm effectively modifies the natural gradient flow, introducing both dissipative damping and skew-symmetric coupling to suppress oscillatory behavior and stabilize the trajectory toward the equilibrium. The complete procedure is detailed in Algorithm~\ref{alg:controlled_minimax}.

\begin{algorithm}[h]
\caption{Controlled-Minimax Optimization}
\label{alg:controlled_minimax}
\begin{algorithmic}[1]
\REQUIRE Objective function $L(\theta, \phi)$, learning rate $\eta > 0$
\REQUIRE Damping gains $\lambda_\theta, \lambda_\phi$, coupling parameter $\beta$
\REQUIRE Initial parameters $\theta_0, \phi_0$

\STATE $k \leftarrow 0$
\WHILE{convergence criteria not met}
    \STATE \textbf{1. Compute Gradients:}
    \STATE $g_\theta \leftarrow \nabla_\theta L(\theta_k, \phi_k)$
    \STATE $g_\phi \leftarrow \nabla_\phi L(\theta_k, \phi_k)$

    \STATE \textbf{2. Compute Control Inputs:}
    % \STATE \textit{Using instantaneous velocities $v_\theta = -g_\theta$ and $v_\phi = g_\phi$:}
    \STATE $u_{\theta} \leftarrow \lambda_\theta g_\theta + \beta A_c g_\phi$
    \STATE $u_{\phi} \leftarrow \beta A_c^\top g_\theta - \lambda_\phi g_\phi$

    \STATE \textbf{3. Update Parameters:}
    \STATE $\theta_{k+1} \leftarrow \theta_k - \eta g_\theta + \eta u_{\theta}$
    \STATE $\phi_{k+1} \leftarrow \phi_k + \eta g_\phi + \eta u_{\phi}$

    \STATE $k \leftarrow k + 1$
\ENDWHILE
\STATE \textbf{return} $\theta_k, \phi_k$
\end{algorithmic}
\end{algorithm}

The proposed Controlled-Minimax Optimization algorithm stabilizes the oscillatory dynamics inherent in standard Gradient Descent-Ascent (GDA) by modeling the optimization as a controlled dynamical system. We introduce two feedback mechanisms to counteract the rotational vector fields that typically prevent convergence to Nash equilibria:

\begin{itemize}
    \item \textbf{Dissipative Damping}: The terms parameterized by $\lambda_\theta$ and $\lambda_\phi$ introduce artificial friction into the update laws. This negative feedback dampens the kinetic energy accumulation that drives trajectories into limit cycles, effectively improving the local conditioning of the optimization landscape.
    \item \textbf{Cross-Coupled Correction}: The terms involving $\beta$ and $A_c$ inject the opponent's gradient information into the current player's update. This cross-coupling applies a corrective force orthogonal to the standard gradient flow, neutralizing the rotational component of the vector field and straightening the trajectory toward the saddle point.
\end{itemize}

Collectively, these controls transform the update laws from simple steepest descent/ascent into a stabilized trajectory that balances local gradient exploitation with improved convergence stability.

\begin{figure*}[t]
    \centering
    \begin{minipage}{0.92\textwidth}
        \begin{subfigure}[b]{0.24\textwidth}
            \centering
            \includegraphics[width=\textwidth]{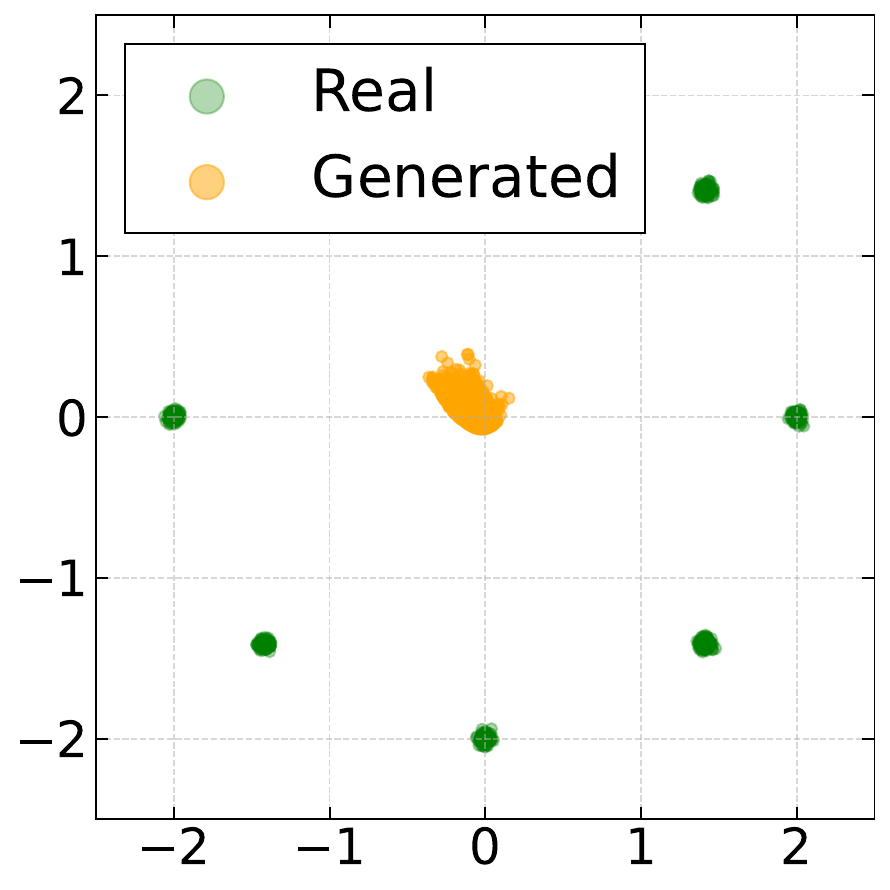}
        \end{subfigure}
        \hfill
        \begin{subfigure}[b]{0.24\textwidth}
            \centering
            \includegraphics[width=\textwidth]{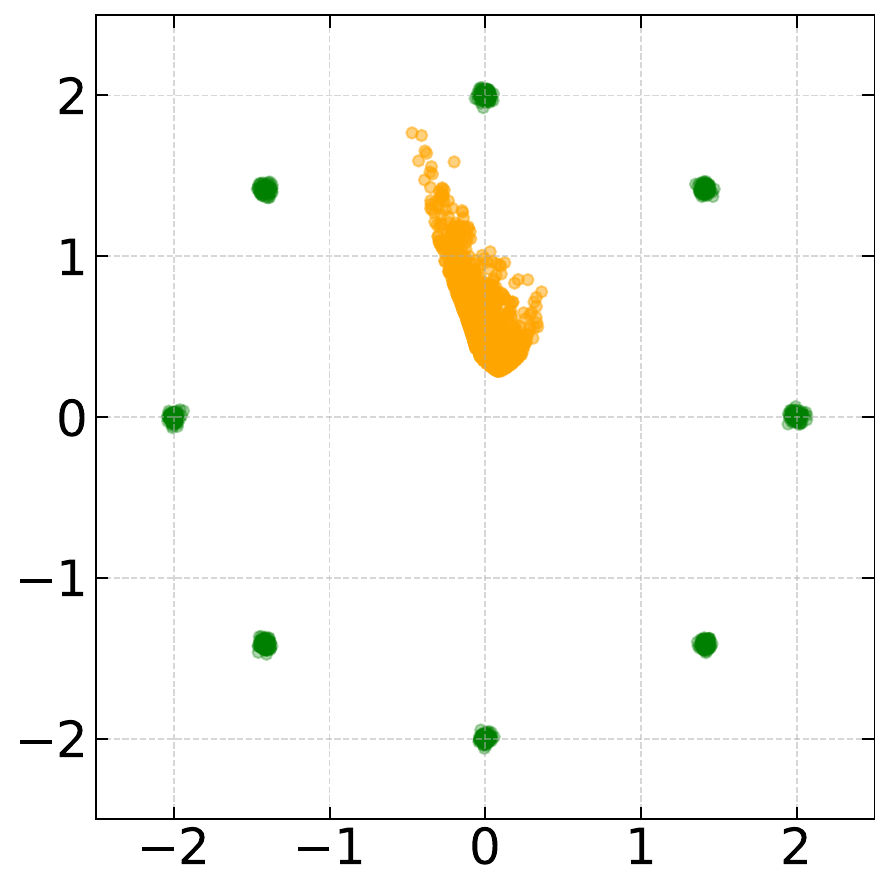}
        \end{subfigure}
        \hfill
        \begin{subfigure}[b]{0.24\textwidth}
            \centering
            \includegraphics[width=\textwidth]{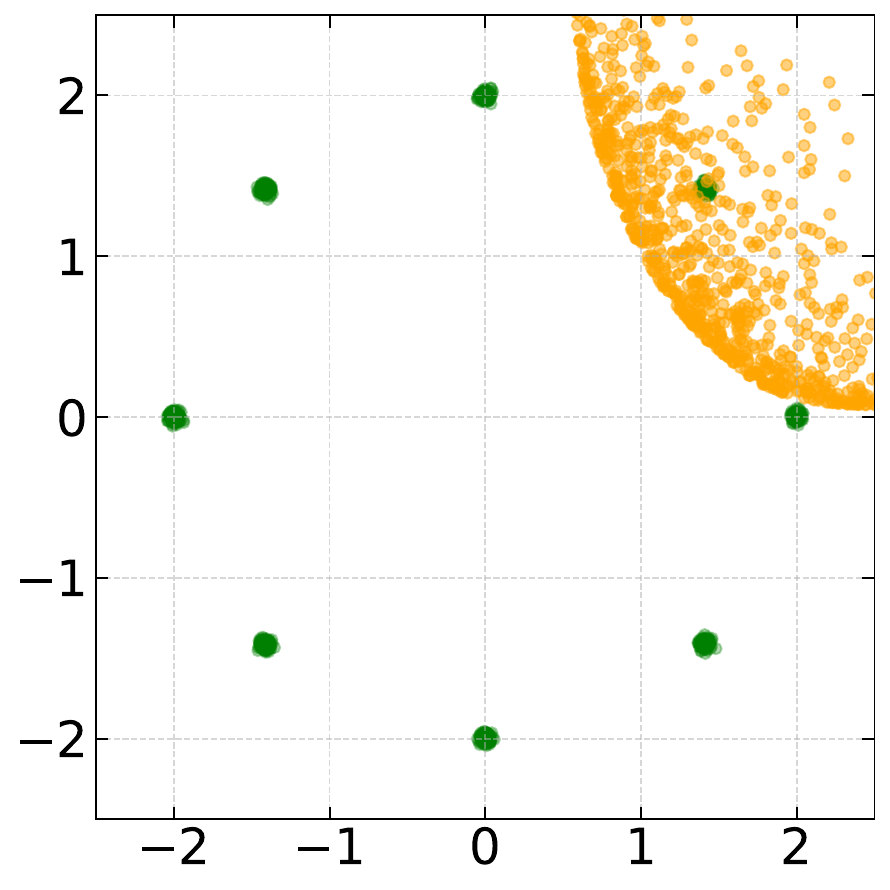}
        \end{subfigure}
        \hfill
        \begin{subfigure}[b]{0.24\textwidth}
            \centering
            \includegraphics[width=\textwidth]{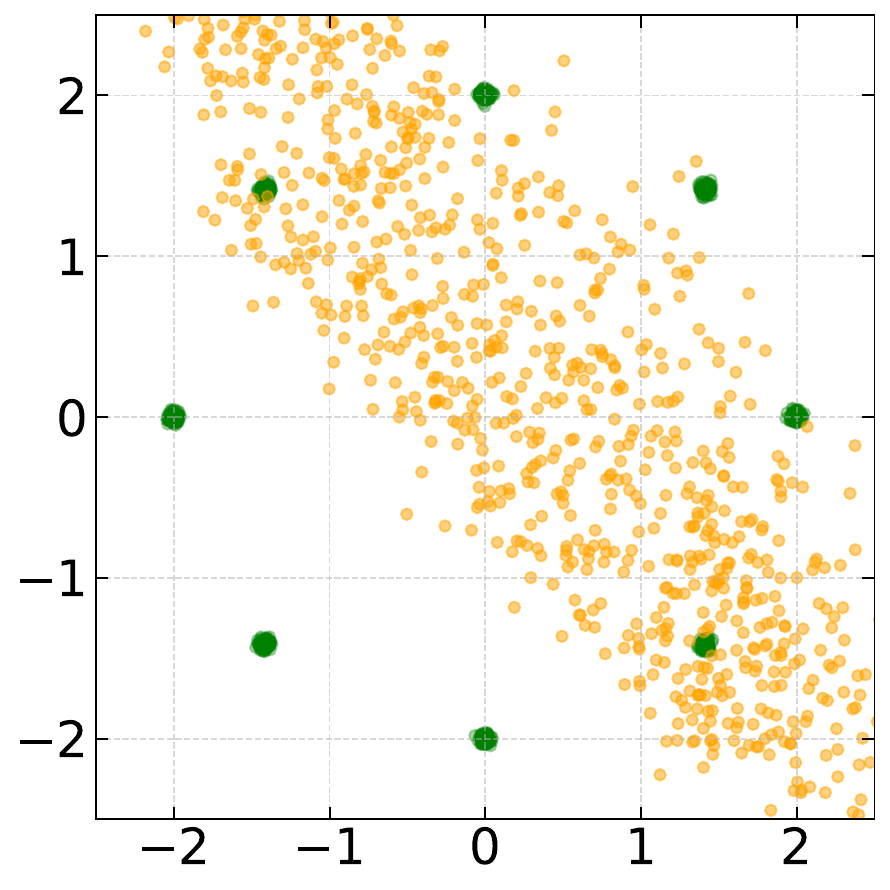}
        \end{subfigure}
    \end{minipage}
    
    \vspace{0.3cm}
    \begin{minipage}{0.92\textwidth}
        \begin{subfigure}[b]{0.24\textwidth}
            \centering
            \includegraphics[width=\textwidth]{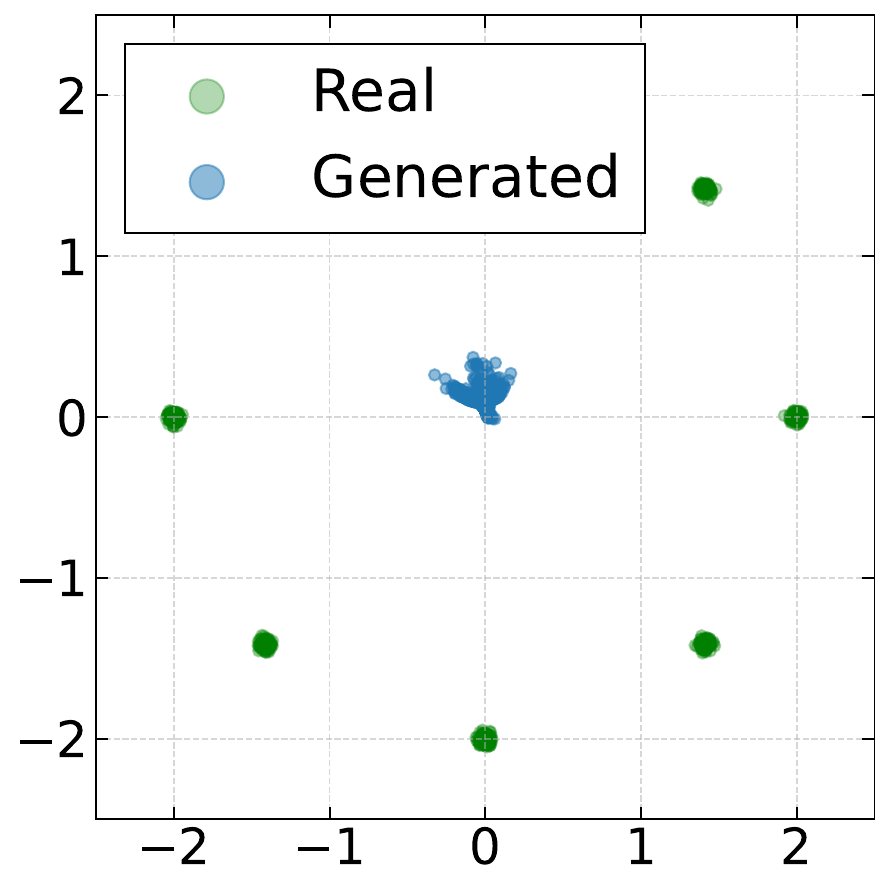}
        \end{subfigure}
        \hfill
        \begin{subfigure}[b]{0.24\textwidth}
            \centering
            \includegraphics[width=\textwidth]{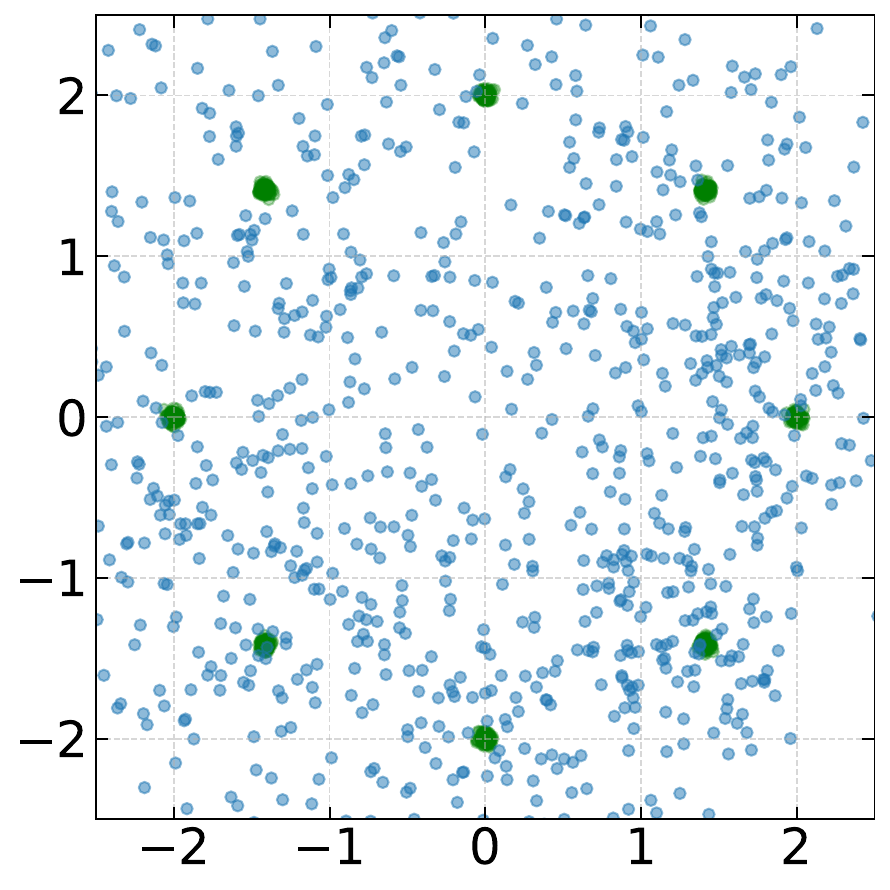}
        \end{subfigure}
        \hfill
        \begin{subfigure}[b]{0.24\textwidth}
            \centering
            \includegraphics[width=\textwidth]{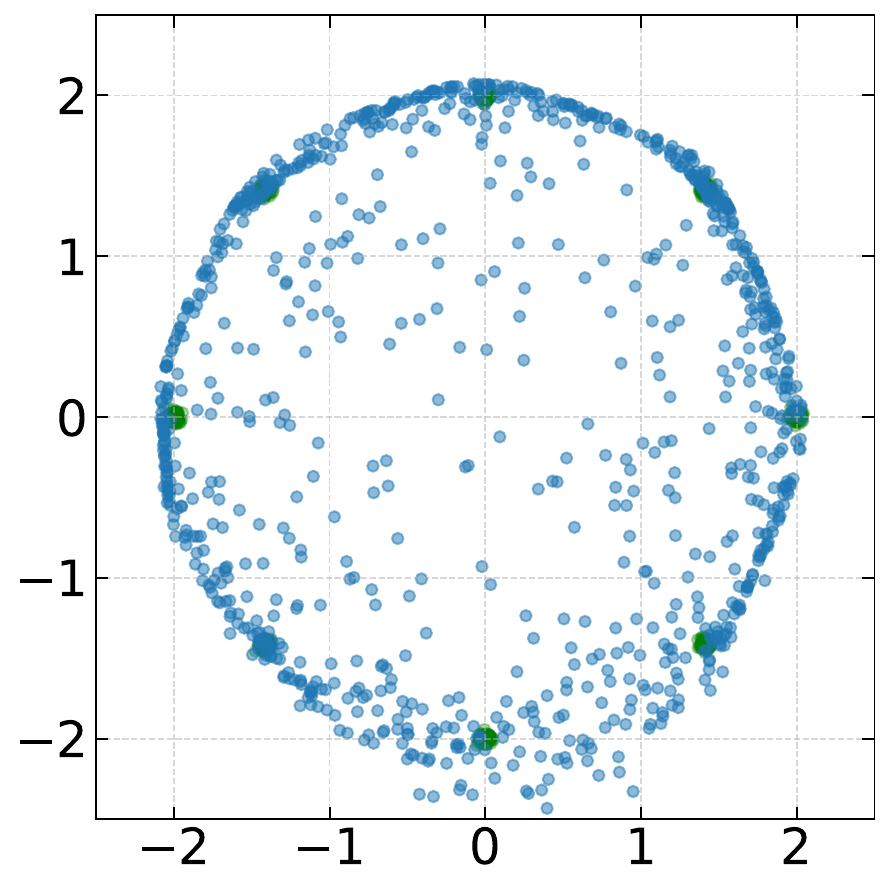}
        \end{subfigure}
        \hfill
        \begin{subfigure}[b]{0.24\textwidth}
            \centering
            \includegraphics[width=\textwidth]{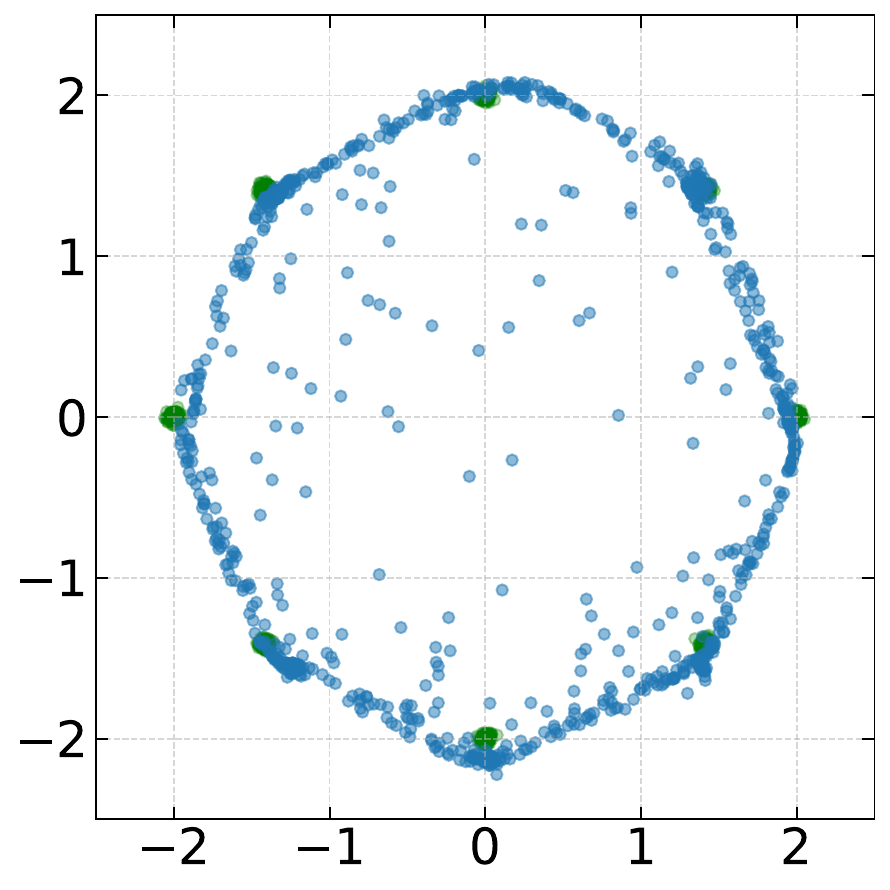}
        \end{subfigure}
    \end{minipage}

    \caption{Training evolution on the 8-Gaussian mixture dataset. The top row displays the distribution generated by the Standard GAN baseline, while the bottom row illustrates the proposed Controlled-Minimax. Columns correspond to training iterations $t={0, 1000, 3000, 4999}$.}
    \label{fig:gan_comparison}
\end{figure*}

\section{Experiments}
In this section, we conduct a comprehensive experimental analysis ranging from low-dimensional pathological examples to high-dimensional image synthesis tasks to empirically validate the theoretical guarantees of the Controlled-Minimax framework. We begin by examining convergence behavior on synthetic benchmarks known to induce rotational dynamics. Specifically, we visualize optimization trajectories and verify local exponential stability on the Dirac-GAN \citep{mescheder2018training} in Figure \ref{fig:combined_plots}, the Swiss Roll manifold in Figure \ref{fig:comparison_swiss}, and the 2D Gaussian mixtures \citep{rasmussen1999infinite} in Figure \ref{fig:gan_comparison}.

Subsequently, we evaluate the scalability and robustness of our method on large-scale adversarial image generation using the StyleGAN2-ada \citep{karras2020training} architecture across three standard datasets of varying complexity: CIFAR-10 \citep{krizhevsky2009learning}, LSUN (Bedroom and Cat categories) \citep{yu2015lsun}, and Flickr-Faces-HQ (FFHQ) \citep{karras2019style}. By employing the Fréchet Inception Distance (FID) score \citep{heusel2017gans} to quantitatively assess sample quality and diversity, we demonstrate that our Controlled-Minimax effectively suppresses oscillatory behavior and accelerates convergence relative to baseline methods without compromising generation fidelity in Figure \ref{fig:comparison_main} and Table \ref{para}.

\subsection{Dirac-GAN}
\textbf{Problem Setting.} The real data distribution is a Dirac-delta function centered at the origin, $p_{data} = \delta_0$. The generator is parameterized by a single scalar $\theta$ and is defined as $G_{\theta}(z) = \theta$. The discriminator is modeled as a linear function $D_{\psi}(x) = \psi \cdot x$, parameterized by a scalar $\psi$. The unique Nash equilibrium for this game is located at the origin $(\theta, \psi) = (0, 0)$.

\textbf{Result.} Figure \ref{fig:combined_plots} presents a comparison of the optimization trajectories in parameter space. As expected, the standard Dirac-GAN baseline exhibits persistent oscillatory behavior, with $\theta$ and $\psi$ cycling indefinitely around the origin without converging. In contrast, our Controlled Minimax method effectively dampens these oscillations, driving the parameters $(\theta, \psi)$ rapidly to the unique Nash equilibrium at $(0,0)$.

\subsection{2D Gaussian}
\textbf{Dataset}. To evaluate the model's ability to capture multimodal distributions and resist mode collapse, we utilize a synthetic 2D dataset consisting of a mixture of eight isotropic Gaussians. The modes are arranged in a circle with a radius of 2.0 around the origin, positioned at intervals of $\pi/4$ (specifically at coordinates $(\pm 2, 0)$, $(0, \pm 2)$, and $(\pm \sqrt{2}, \pm \sqrt{2})$). Each Gaussian mode has a standard deviation of $\sigma = 0.02$, creating distinct, separated clusters with minimal overlap. During training, real data samples are drawn uniformly from these eight mixture components, while the generator attempts to map a standard normal latent distribution to this target manifold.

\textbf{Network Architecture.} Both the generator ($G$) and the discriminator ($D$) are parameterized as fully connected Multi-Layer Perceptrons (MLPs). The generator maps a 2-dimensional latent vector $z \sim \mathcal{N}(0, I)$ to the 2-dimensional data space using a network with two hidden layers, each containing 128 units. Rectified Linear Unit (ReLU) activation functions are applied after each hidden layer, while the output layer is linear. The discriminator mirrors this architecture, accepting 2-dimensional data points as input and processing them through two hidden layers of 128 units with ReLU activations. The final layer of the discriminator projects the features to a scalar scalar value, which is passed through a Sigmoid activation function to output the probability of the input being real.

\textbf{Result.} The qualitative results of the 2D Gaussian experiment are shown in Figure \ref{fig:gan_comparison}. The standard GAN baseline exhibits severe mode collapse; despite the target distribution containing eight distinct modes, the generator fails to capture the full multimodal structure. In contrast, our Controlled Minimax approach successfully stabilizes the training dynamics. The generator learns to distribute its probability mass evenly across all eight Gaussian components, accurately recovering the geometry and diversity of the underlying distribution without dropping modes.

\begin{figure*}[t] % The '*' makes it span both columns. Remove '*' for a single column.
    \centering
    
    % --- First Subfigure ---
    \begin{subfigure}[b]{0.33\textwidth}
        \centering
        \includegraphics[width=\linewidth]{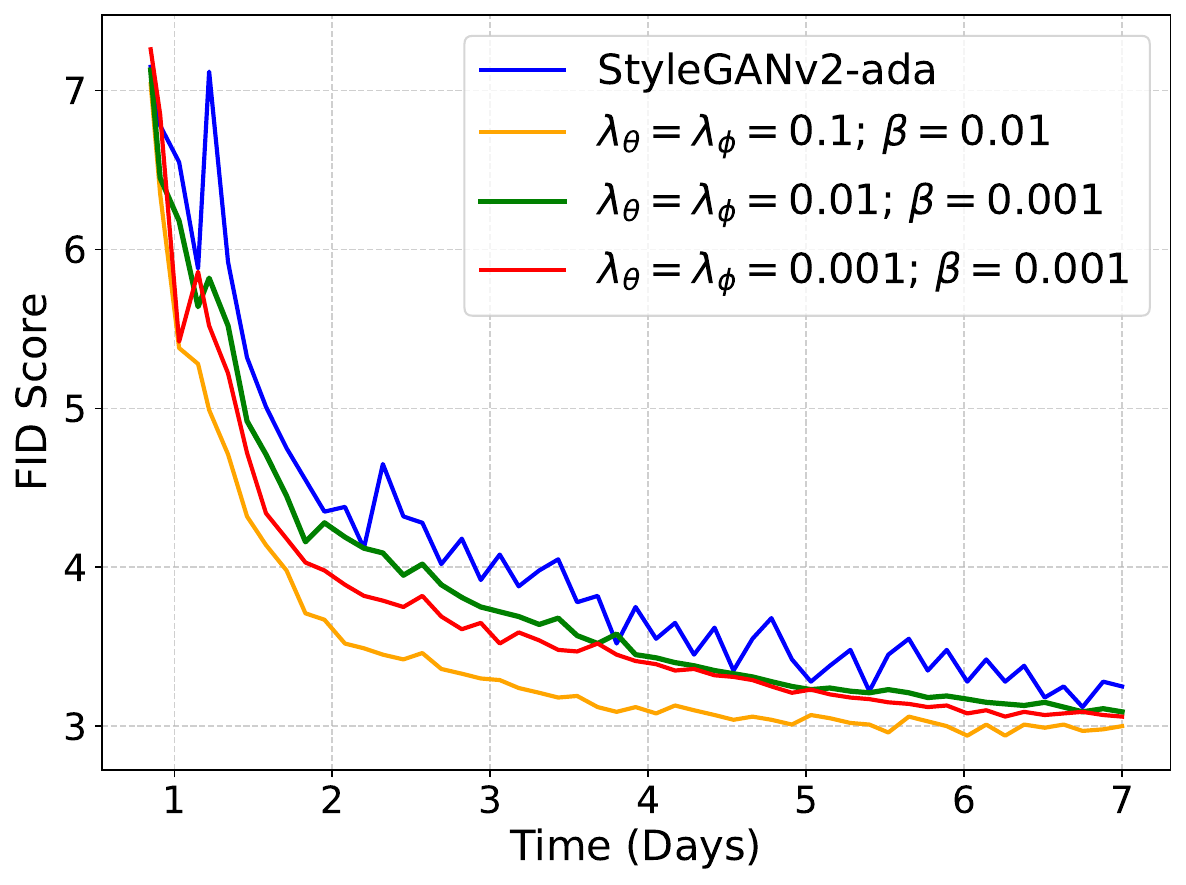}
        \caption{Cifar 10}
        \label{fig:target}
    \end{subfigure}
    \hfill % Adds flexible space between images
    % --- Second Subfigure ---
    \begin{subfigure}[b]{0.33\textwidth}
        \centering
        \includegraphics[width=\linewidth]{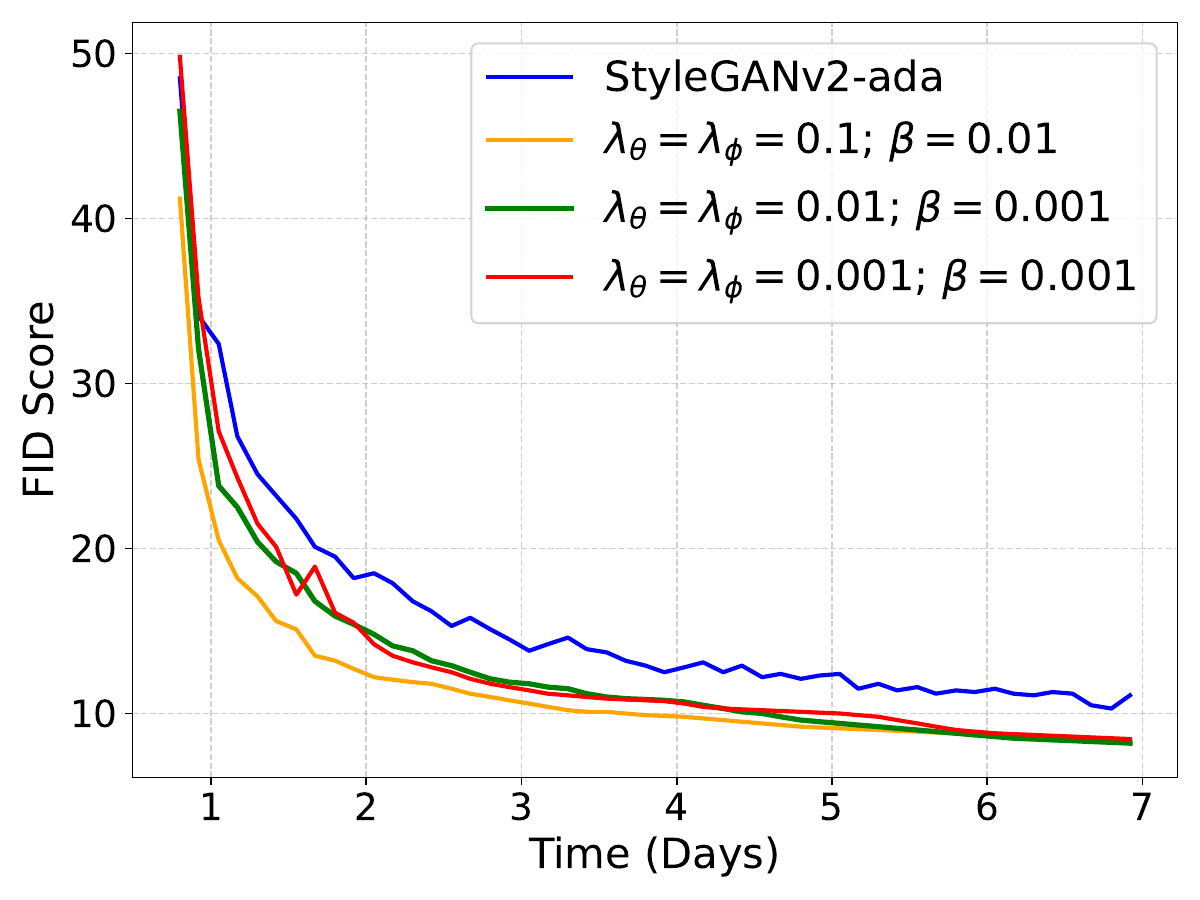}
        \caption{LSUN-Cat}
        \label{fig:baseline}
    \end{subfigure}
    \hfill % Adds flexible space between images
    % --- Third Subfigure ---
    \begin{subfigure}[b]{0.33\textwidth}
        \centering
        \includegraphics[width=\linewidth]{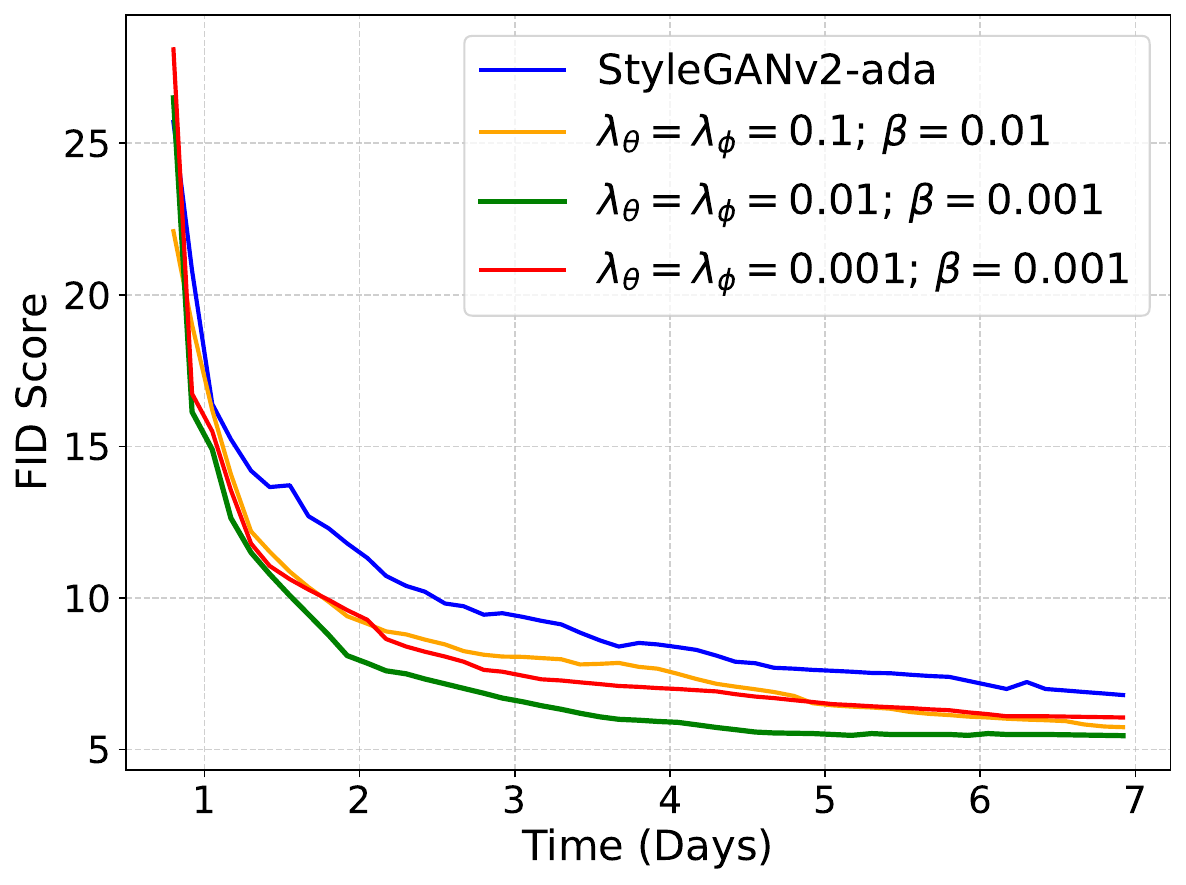}
        \caption{FFHQ}
        \label{fig:ours}
    \end{subfigure}
    
    \caption{Training convergence analysis across different datasets. We plot the FID score as a function of training time (in days). The curves compare the performance of various hyperparameter configurations against the StyleGAN2-ada baseline.}
    \label{fig:comparison_main}
\end{figure*}

\subsection{Swiss Roll}

\textbf{Dataset.} 
We evaluate our method on the 2D Swiss Roll dataset, a classic benchmark for manifold learning. The target distribution $p_{data}(x)$ is generated by sampling $N=3,000$ points according to $x = t \cos(t)$ and $y = t \sin(t)$, where $t \sim \mathcal{U}[1.5\pi, 4.5\pi]$. We add Gaussian noise $\epsilon \sim \mathcal{N}(0, 0.25^2)$ to simulate real-world variance. 

\textbf{Network Architectures.} 
To ensure a fair comparison, both the baseline and our method share identical backbones. The generator has an MLP with 3 hidden layers (128 units, ReLU, Batch Norm) mapping $z \in \mathbb{R}^2 \sim \mathcal{N}(0, I)$ to data space, while the discriminator has an MLP with 3 hidden layers (128 units, LeakyReLU $\alpha=0.2$) outputting a scalar probability.

\textbf{Result.} The generated distributions for the Swiss Roll dataset are visualized in Figure \ref{fig:comparison_swiss}. The Baseline GAN struggles to learn the continuous, spiraling manifold, and exhibits mode collapse. Conversely, our Controlled Minimax method successfully generates samples that closely adhere to the true spiral trajectory, maintaining uniform coverage along the manifold and effectively reproducing the noise characteristics of the target distribution.

\subsection{StyleGAN2-ada}
\textbf{Setting.} We strictly adhere to the experimental settings described in the original paper \cite{karras2020training} to ensure a direct and fair comparison. We evaluate performance on three standard benchmarks: CIFAR-10 ($32 \times 32$), LSUN Cat ($256 \times 256$), and FFHQ ($1024 \times 1024$).

\textbf{Result.} We compare the FID score of our Controlled-Mimimax against the standard StyleGAN2-ada baseline. Figure \ref{fig:comparison_main} illustrates the convergence speed, plotting FID scores against training time on various hyperparameter choices. Our method demonstrates faster convergence and greater stability across all three datasets. Table \ref{para} presents the quantitative evaluation after 7 days of training. Our approach consistently outperforms the baseline, achieving the lowest FID scores across all tested resolutions. Specifically, on Cifar10, our optimal configuration yields an FID of $2.91$, a marked improvement over the baseline's $3.37$. This performance advantage is further amplified in high-resolution domains; on LSUN Cat and FFHQ, we reduce the FID from $10.48$ to $8.18$ and from $7.81$ to $5.46$, respectively. Furthermore, the reduced oscillation ranges observed over the final 24 hours demonstrate superior stability, contrasting sharply with the high variance exhibited by the baseline near convergence. These findings serve to validate our theoretical stability guarantees, demonstrating that the Controlled-Minimax framework effectively improves both model performance and convergence behavior.

\begin{table}[h]  
\begin{adjustbox}{width={82mm},keepaspectratio}%
\begin{tabular}{|c|c|clc|}
\hline
\multicolumn{2}{|c|}{} &\multicolumn{1}{c|}{$\lambda = 0.1$} & \multicolumn{1}{l|}{$\lambda = 0.01$} &  $\lambda = 0.001$ \\ \hline

\multirow{4}{*}{\begin{tabular}[c]{@{}c@{}}Cifar10\\ (32*32)\end{tabular}}  & $\beta = 0.01$ & $ 2.98 \pm 0.03$& $3.22 \pm 0.02$ & $3.45 \pm 0.05$\\ 
  &  $\beta = 0.001$ &  $\mathbf{2.91 \pm 0.02}$   &  $3.19\pm0.03$  & $3.29 \pm 0.04$  \\
 &   $\beta = 0.0001$  &  $3.01 \pm 0.03$ & $2.94 \pm 0.02$  & $3.14 \pm 0.05$ \\ \cline{2-5} 
  & Baseline     & \multicolumn{3}{c|}{$3.37 \pm 0.24$}             \\ \hline

\multirow{4}{*}{\begin{tabular}[c]{@{}c@{}}LSUN\\Cat\\ (256*256)\end{tabular}}     &  $\varrho_{2} = 0.01$   &    $8.45 \pm 0.23$       &           $9.21 \pm 0.32$         &  $ 9.84 \pm 0.26$\\ 
  &   $\varrho_{2} = 0.001$  & $8.39 \pm 0.12$      &  $8.24 \pm 0.27$          &   $9.92 \pm 0.25$\\ 
  & $\varrho_{2} = 0.0001$  &     $8.73 \pm 0.17$   &    $\mathbf{8.18 \pm 0.16}$        & $8.42 \pm 0.36$ \\ \cline{2-5} 
 & Baseline     & \multicolumn{3}{c|}{$10.48 \pm 0.44$}             \\ \hline
 
\multirow{4}{*}{\begin{tabular}[c]{@{}c@{}}FFHQ\\ (1024*1024)\end{tabular}} & $\varrho_{2} = 0.01$ &     $5.75 \pm 0.18$       &    $6.85 \pm 0.29$            & $7.93\pm 0.27$ \\
&$\varrho_{2} = 0.001$  & $5.98 \pm 0.29$ &    $\mathbf{5.46 \pm 0.16}$    & $6.83 \pm 0.35$  \\ 
&$\varrho_{2} = 0.0001$  &  $6.84 \pm 0.31$  & $6.57 \pm 0.23$        &  $6.06 \pm 0.28$\\ \cline{2-5}
& Baseline     & \multicolumn{3}{c|}{$7.81 \pm 0.39$}                \\ \hline
\end{tabular}
\end{adjustbox}
\caption{FID scores for each dataset after training for 7 days. ($a \pm b$ in this table should be interpret as $a$ is the FID score after 7 days, $b$ is the range of oscillation on the 24 hours span of the seventh day). Throughout, we assume a unified parameter $\lambda = \lambda_\theta = \lambda_\phi$.}
\label{para}
\end{table}

% \begin{table}[h]
% \centering
% \caption{Final FID scores. Results are reported as mean $\pm$ standard deviation over 10 runs. Lower is better.}
% \label{tab:fid_scores}
% \begin{tabular}{lcc}
% \toprule
% \textbf{Dataset} & \textbf{StyleGAN2-ADA} & \textbf{Ours} \\
% \midrule
% CIFAR-10    & $2.92 \pm 0.05$ & $\mathbf{2.65 \pm 0.04}$ \\
% LSUN Cat & $3.86 \pm 0.12$ & $\mathbf{3.41 \pm 0.09}$ \\
% FFHQ        & $2.70 \pm 0.08$ & $\mathbf{2.45 \pm 0.06}$ \\
% \bottomrule
% \end{tabular}
% \end{table}

\section{Conclusion and Discussion}
In this work, we have addressed the fundamental instability of minimax optimization through the lens of continuous-time dynamical systems. We identified that the persistent limit cycles and divergence often observed in adversarial learning stem from the problematic eigenstructure of the loss function’s Hessian near equilibrium points. To overcome these inherent structural limitations, we introduced the Controlled-Minimax framework. Our approach leverages control theory to derive a specific feedback controller that actively steers the optimization trajectory. 

Theoretical analysis confirms that this framework guarantees local exponential stability regardless of the underlying Hessian spectrum, effectively stabilizing problems where standard gradient descent-ascent methods fail. Empirically, our extensive validation ranging from pathological synthetic benchmarks like the Dirac-GAN, 2D Gaussian and Swiss Roll to high-dimensional image synthesis with StyleGAN2-ada demonstrates that our Controlled-Minimax not only suppresses oscillatory behavior but also ensures robust and rapid convergence. Collectively, these results bridge the gap between control theory and deep learning, offering a principled solution to one of the most persistent challenges in adversarial training.

Several promising directions remain for future exploration. The current framework relies on fixed control parameters which guarantee local stability; however, designing an adaptive controller that adjusts to the local geometry could offer a pathway toward global stability guarantees. Furthermore, acknowledging the inherent discrepancy between continuous ODE models and iterative algorithms, future work should rigorously address discrete-time stability. Establishing stability conditions directly for the discretized update rules will be essential to fully reconcile the theoretical control perspective with practical algorithmic behavior.
\nocite{langley00}

\section*{Broader Impact}

This paper presents work whose goal is to advance the field of machine learning. There are many potential societal consequences of our work, none of which we feel must be specifically highlighted here.

%%%%%%%%%%%%%%%%%%%%%%%%%%%%%%%%%%%%%%%%%%%%%%%%%%%%%%%%%%%%%%%%%%%%%%%%%%%%%%%
%%%%%%%%%%%%%%%%%%%%%%%%%%%%%%%%%%%%%%%%%%%%%%%%%%%%%%%%%%%%%%%%%%%%%%%%%%%%%%%
% APPENDIX
%%%%%%%%%%%%%%%%%%%%%%%%%%%%%%%%%%%%%%%%%%%%%%%%%%%%%%%%%%%%%%%%%%%%%%%%%%%%%%%
%%%%%%%%%%%%%%%%%%%%%%%%%%%%%%%%%%%%%%%%%%%%%%%%%%%%%%%%%%%%%%%%%%%%%%%%%%%%%%%
\newpage
\appendix
\onecolumn
\section{Proof of Theorem \ref{thm:local-stability}}
\begin{proof}
Consider the quadratic Lyapunov function
\[
    V(v) = \frac12\big(\|\dot\theta\|^2+\|\dot\phi\|^2\big)
         = \frac12\|v\|^2.
\]
Along trajectories of the linear system~\eqref{eq:velocity-linear} we have
\[
    \dot V = v^\top \dot v = v^\top M v.
\]
Writing $v=(\dot\theta,\dot\phi)$ in block form yields
\begin{align}
    \dot V
    &=
    \dot\theta^\top(-H_{\theta\theta}^\ast-\lambda_\theta I)\dot\theta
    +\dot\theta^\top(-H_{\theta\phi}^\ast+\beta A_c)\dot\phi \nonumber\\
    &\quad
    +\dot\phi^\top(H_{\phi\theta}^\ast-\beta A_c^\top)\dot\theta
    +\dot\phi^\top(H_{\phi\phi}^\ast-\lambda_\phi I)\dot\phi.
    \label{eq:dV-expand}
\end{align}

\emph{Diagonal terms.}
By Assumption~\ref{ass:bounded-hessian}, and using the induced matrix norm, we have
\[
    \|\dot\theta^\top H_{\theta\theta}^\ast\dot\theta\|
    \le \|H_{\theta\theta}^\ast\|\,\|\dot\theta\|^2
    \le L\|\dot\theta\|^2,
\]
and similarly for the $\phi$-block. Hence
\begin{align*}
    \dot\theta^\top(-H_{\theta\theta}^\ast-\lambda_\theta I)\dot\theta
    &\le (L-\lambda_\theta)\|\dot\theta\|^2,\\
    \dot\phi^\top(H_{\phi\phi}^\ast-\lambda_\phi I)\dot\phi
    &\le (L-\lambda_\phi)\|\dot\phi\|^2.
\end{align*}

\emph{Cross terms.}
Define the cross contribution
\[
    T_{\text{cross}}
    := \dot\theta^\top(-H_{\theta\phi}^\ast + \beta A_c)\dot\phi
       + \dot\phi^\top(H_{\phi\theta}^\ast - \beta A_c^\top)\dot\theta.
\]
Expanding gives
\[
    T_{\text{cross}}
    = - \dot\theta^\top H_{\theta\phi}^\ast \dot\phi
      + \beta \dot\theta^\top A_c \dot\phi
      + \dot\phi^\top H_{\phi\theta}^\ast \dot\theta
      - \beta \dot\phi^\top A_c^\top \dot\theta.
\]
Since $\dot\phi^\top A_c^\top\dot\theta = \dot\theta^\top A_c\dot\phi$, 
the antisymmetric controller terms cancel exactly, i.e.,
\[
    \beta\dot\theta^\top A_c\dot\phi
    - \beta\dot\phi^\top A_c^\top\dot\theta = 0,
\]
and hence
\[
    T_{\text{cross}}
    = - \dot\theta^\top H_{\theta\phi}^\ast \dot\phi
      + \dot\phi^\top H_{\phi\theta}^\ast \dot\theta.
\]

Again using Assumption~\ref{ass:bounded-hessian} and Cauchy--Schwarz,
\[
    |T_{\text{cross}}|
    \le \big(\|H_{\theta\phi}^\ast\|+\|H_{\phi\theta}^\ast\|\big)
        \|\dot\theta\|\,\|\dot\phi\|
    \le 2L\,\|\dot\theta\|\,\|\dot\phi\|.
\]

\emph{Combining bounds.}
Substituting these estimates into~\eqref{eq:dV-expand}, we obtain
\[
    \dot V
    \le (L-\lambda_\theta)\|\dot\theta\|^2
       + (L-\lambda_\phi)\|\dot\phi\|^2
       + 2L\|\dot\theta\|\,\|\dot\phi\|.
\]
Using the inequality $2ab\le a^2+b^2$ for all $a,b\ge 0$, we bound
\[
    2L\|\dot\theta\|\,\|\dot\phi\|
    \le L\|\dot\theta\|^2 + L\|\dot\phi\|^2,
\]
and thus
\[
    \dot V
    \le (2L-\lambda_\theta)\|\dot\theta\|^2
       + (2L-\lambda_\phi)\|\dot\phi\|^2.
\]
If~\eqref{eq:lambda-condition} holds, then
\[
    \gamma := \min\{\lambda_\theta-2L,\ \lambda_\phi-2L\} > 0,
\]
and we have
\[
    \dot V \le -\gamma\big(\|\dot\theta\|^2+\|\dot\phi\|^2\big)
             = -2\gamma V.
\]
Integrating yields
\[
    V(t) \le V(0)e^{-2\gamma t},\qquad
    \|v(t)\| \le e^{-\gamma t}\,\|v(0)\|,
\]
which shows exponential stability of $v=0$ for the linear system~\eqref{eq:controlled-system}.

For the nonlinear system~\eqref{eq:controlled-system}, the closed-loop dynamics are $C^1$ and locally Lipschitz by Assumption~\ref{ass:well-posed} and the linear controller~\eqref{eq:controller}. The Jacobian of the nonlinear system at $(\theta^\ast,\phi^\ast,0,0)$ coincides with the matrix $M$ in~\eqref{eq:M-matrix}. Since the linearization is exponentially stable, standard nonlinear stability theory implies that the equilibrium $(\theta^\ast,\phi^\ast,0,0)$ of the nonlinear system is locally exponentially stable. 
\end{proof}

% You can have as much text here as you want. The main body must be at most $8$
% pages long. For the final version, one more page can be added. If you want, you
% can use an appendix like this one.

% The $\mathtt{\backslash onecolumn}$ command above can be kept in place if you
% prefer a one-column appendix, or can be removed if you prefer a two-column
% appendix.  Apart from this possible change, the style (font size, spacing,
% margins, page numbering, etc.) should be kept the same as the main body.
%%%%%%%%%%%%%%%%%%%%%%%%%%%%%%%%%%%%%%%%%%%%%%%%%%%%%%%%%%%%%%%%%%%%%%%%%%%%%%%
%%%%%%%%%%%%%%%%%%%%%%%%%%%%%%%%%%%%%%%%%%%%%%%%%%%%%%%%%%%%%%%%%%%%%%%%%%%%%%%

\end{document}